\documentclass[conference,compsoc]{IEEEtran}

\ifCLASSOPTIONcompsoc
  \usepackage[nocompress]{cite}
\else
  \usepackage{cite}
\fi

\ifCLASSINFOpdf
\else
\fi

\newcommand{\al}{\alpha}
\newcommand{\si}{\sigma}
\newcommand{\E}[1]{\mathbb{E}\left[#1\right]}
\usepackage{amsmath, amssymb}
\usepackage{amsthm}
\newcommand{\Prob}{\mathbb{P}}

\newtheorem{lemma}{Lemma}

\newtheorem{prop}{Proposition}
\newtheorem{remark}{Remark}
\newtheorem{corr}{Corollary}
\newtheorem{ex}{Example}

\usepackage[usenames,dvipsnames]{xcolor}

\usepackage{graphicx}

\usepackage{algorithm}
\usepackage{algorithmic}
\usepackage{hyperref}
\usepackage[capitalise]{cleveref}

\usepackage{booktabs}

\begin{document}
\title{Guaranteed Fixed-Confidence Best Arm Identification in Multi-Armed Bandits: Simple Sequential Elimination Algorithms}

\author{\IEEEauthorblockN{MohammadJavad Azizi}
\IEEEauthorblockA{
University of Southern California\\
azizim@usc.edu}
\and
\IEEEauthorblockN{Sheldon Ross}
\IEEEauthorblockA{University of Southern California\\
smross@usc.edu}
\and
\IEEEauthorblockN{Zhengyu Zhang}\IEEEauthorblockA{University of Southern California\\
zhan892@usc.edu}}

\maketitle
\thispagestyle{plain}
\pagestyle{plain}

\begin{abstract}
We consider the problem of finding, through adaptive sampling, which of n options (arms) has the largest mean.  Our objective is to determine a rule which identifies the best arm with a fixed minimum confidence using as few observations as possible, i.e. this is a fixed-confidence (FC) best arm identification (BAI) in multi-armed bandits. We study such problems under the Bayesian setting with both Bernoulli and Gaussian arms. We propose to use the classical \emph{vector at a time} (VT) rule, which samples each remaining arm once in each round. We show how VT can be implemented and analyzed in our Bayesian setting and be improved by early elimination. Our analysis show that these algorithms guarantee an optimal strategy under the prior. We also propose and analyze a variant of the classical \emph{play the winner} (PW) algorithm. Numerical results show that these rules compare favorably with state-of-art algorithms.
\end{abstract}

\IEEEpeerreviewmaketitle

\section{Introduction and Formulation}\label{sec:intro}%
Let $F_{\theta}(x)$ be a family of distributions indexed by its mean  $\theta.$
Suppose there are $n$ arms, and each new observation from arm $i$ is a random variable independent of previous observations with distribution $F_{\theta_i },$ where $\theta_1, \ldots, \theta_n$ are unknown. We want to find which arm has the largest mean. A decision is made at each stage as to which arm to sample next from (sampling rule), when to stop (stopping rule) and declare which arm has the largest mean (recommendation rule). The objective is to minimize number of samples, $N$, subject to the condition that the probability of correct choice is at least $\alpha$. This setting is mostly known as \emph{Fixed confidence best arm identification} in the reinforcement learning literature. We study such models both when the arm distributions are Bernoulli and Gaussian with a fixed variance.

Best arm identification has many applications. Foremost is probably in clinical trials to determine which of several medical approaches (e.g., drugs, treatments, procedures, and various vaccines) yields the best performance. A arm here refers to a particular approach, with its use resulting in either a success (suitably defined) or not. Online advertising is another application \cite{Li_2010LinUCB, DisplayAdvertising2017}, where a decision maker is trying to decide which of $n$ different advertisements to utilize. For instance, the advertisements might be a recruitment ad, and a success might refer to a subsequent clicking on the advertisement. Another application is to choose among different methods for teaching a particular skill. Each day, a method can be used on a group of students, with the students being tested at the end of the day with each test resulting in a score which would be pass (1) or fail (0) in the Bernoulli case, and numerical in the Gaussian case.  

This problems has been studied for quite some time. However, the early work, 
had the assumption that the difference between the largest and second largest arm mean is at least some known positive value \cite{bks, bk, h, hm, p, sw1, sw2}. More recent works such as \cite{abm, jmnb, ggl} do not make this assumption, while others like \cite{emm, gk, ru} keep this assumption. 
We take a Bayesian approach that supposes the unknown means are the respective values of independent and identically distributed (i.i.d) random variables 
having a specified \emph{prior} distribution $F$.  In addition, we show that even though we assume a Bayesian setting, our rules can be applied in the Bernoulli case without this assumption, by using a Bayesian rule when the prior is uniform $(0, 1)$, as in the well known \emph{Thompson sampling} rule. The main motivation for our work is the observation that many recent algorithms for FC BAI are over conservative in their stopping rule \cite{degenne2020gamification, wang2021core, bastani2020explorationfree}; this is to say their accuracy is much bigger than desired confidence $\al$ which causes $\E{N}$ be very large. This is because they are designed based on a worst-case scenario \cite{wang2021core}. Given the prior, we know that these worst cases happen with a very small probability so these algorithms are extra conservative. \cref{Fig:PcNormal} and \cref{Fig:EnNormal} show a sample of these algorithms where their accuracy is almost 1 however they have large sampling complexity. This is while our algorithms satisfy the desired confidence level with much smaller $N$. We use Markov chain properties, specially the Gambler's ruin problem \cite{ro} to do the analysis. To the best of our knowledge this is the first time BAI is treated directly using these techniques.

In \cref{sec:VT-Bern}, we consider the Bernoulli case. 
We reconsider the classical "vector at a time" (VT)  rule, with critical value $k\in \mathbb{Z^+}$ as defined in \cref{alg:VT}.
The appropriate value of $k$ that results in confidence of at least $\al$ was determined in \cite{sw1} and  \cite{sw2} under the assumption that the optimality gap of the second best arm, call it $\Delta_{\min}$, is strictly positive. We show how this rule can be implemented and analyzed in our Bayesian setting. We develop a modification of the VT rule by allowing for early elimination in \cref{sec:VT-EE}. 
We show how to determine the probability that the best arm is eliminated early, as well as the mean number of the non-best arms that are eliminated early.

In \cref{sec:PW} we consider a variant of another classical rule: "play the winner" (PW).
This variant differs from the classical model in \cite{sw1} which for instance eliminated a arm if at some point - not necessarily at the end of a round -  it has $k$ fewer successes than another arm for Bernoulli distribution. We show how to analyze this rule in the Bayesian setting. 

We consider the case of Normal arms in \cref{sec:normal}, 
where arm distributions are all Normal with fixed variance $\si^2$ and standard Normal for the prior distribution of the means. Finally in \cref{sec:exps} we present the numerical experiments that confirm our approximations in the analysis and indicate that our algorithm outperform recently proposed algorithms.

\section{VT Rule, Bernoulli Case}\label{sec:VT-Bern}

Suppose there are $n$ Bernoulli arms with respective means $p_1, \ldots, p_n$. As noted earlier, we suppose that $p_1, \ldots, p_n$ are the respective values of iid random variables 
having a specified distribution $F$.
Let $s_i(t)$ show the cumulative number of success for arm $i$ at time $t$. We might use $s_i$ as well when the $t$ is given in the context.
The vector at a time (VT) rule, introduced in \cite{bks}, 
for a given critical value $k\in \mathbb{Z}^+$, is given in \cref{alg:VT}.
\begin{algorithm}[tb]
  \caption{Vector at a time (VT)}
  \label{alg:VT}
\begin{algorithmic}
  \STATE \textbf{Input}: all the arms.
  \REPEAT
  \STATE - Sample each remaining arm once.
  \STATE - For each arm $i$ if $\exists j\neq i\  \text{s.t.}\ s_i(t)\leq s_j(t)-k$ eliminate arm $i$.
  \UNTIL{there is one arm remaining}
\end{algorithmic}
\end{algorithm}

    Note that VT algorithm is similar to Action Elimination in \cite{Jamieson-2014} while it
    is slightly different in that it is non-adaptive and more suitable for our setting. Also in this paper we specifically analyzed VT for Bernoulli and Gaussian distribution and did not use a general estimator of the mean, despite the Action Elimination algorithm.  
	Let  $p_{[1]} > p_{[2]} > \ldots > p_{[n]}$ be the ordered values of the unknown means $p_1, \ldots, p_n,$ and let $C$ be the event that the correct choice is made. 
	We now show how to determine $k$ such that $\Prob(C) \geq  \al$ when $p_1, \ldots, p_n$ are the values of independent and identically distributed random variables 
	$P_1, \ldots, P_n$ 
	having distribution $F$. We let $I\{A\}$ be the indicator of the event $A$, and $X =_{st} Y$ indicates that $X$ and $Y$ have the same distribution. Let $U, U_1, \ldots, U_{n}$ be independent uniform $(0, 1)$ random variables then. 
	\begin{lemma}\label{lemma:max_Xi}
	$\quad \max_i P_i =_{st} F^{-1}(U^{1/n})$\end{lemma} 
	\begin{proof}
	\begin{align*}
	    \max_i P_i &=_{st} \max (F^{-1}(U_1), \ldots, F^{-1}(U_n) )\\ 
		&= F^{-1}(\max_{i=1, \ldots, n} U_i ) \\
		&=_{st} F^{-1} (U^{1/n} )
	\end{align*}
	\end{proof}
	Now suppose that in each round we take an observation from each arm, even those that are eliminated. Let $0$ be the best arm, namely the one with the largest mean, and randomly number the others as arm $1, \ldots, n-1.$ Imagine that the best arm is playing a "Gambler's Ruin Game" with each of the others, with the best one beating arm $i$ if the difference $s_0(t)-s_i(t)$ hits the value $k$ before $-k$. Let $B_i$ be the event that the best arm beats arm $i$ and note that VT algorithm recommends the best arm correctly if it wins all of its games. That is, if we let $B \equiv B_1B_2 \cdots B_{n-1}$
	then $B \subset C,$ thus $\Prob(C) \geq \Prob(B)$. We can bound $\Prob(B)$ as follows.
	\begin{lemma}
	$\Prob(B) \geq  (\Prob(B_1))^{n-1}$
	\end{lemma}%
	\begin{proof}
	Let $U_0, U_1, \ldots, U_{n-1}$ and $U_{i,j}, i = 0, 1, \ldots, n-1, j \geq 1$ all be independent uniform
	$(0, 1)$ random variables.
	We set $P_0 = F^{-1}(U_0^{1/n}),\; P_i = F^{-1}((1-U_i) U_0^{1/n}), i>0.$ We know that conditional on the maximum of $n$ independent uniform $(0, 1)$ random variables, call it $U_{\max}$, the other $n-1$ of them 
	are distributed as independent uniform $(0, U_{\max})$. Now using \cref{lemma:max_Xi}, it follows that the joint distribution of $P_0, P_1, \ldots, P_{n-1}$ is exactly that of the mean of the best arm, followed by the means of the other $n-1$ arms in a random order.
	
	Let $I_{0,j} = I\{1-U_{0,j} < P_0\}, j \geq 1,$ and $I_{i,j} = I\{U_{i,j} < P_i\}, i < n, j \geq 1.$ Note that
	$I_{i,j}$ has the distribution of the $j^{th}$ observation of arm $i$ and also that $I_{0,j}$ is increasing in $U_{0,j}$ whereas $I_{i,j}$ is decreasing in $U_{i,j}, i \geq 1.$ Because $P_i$ is decreasing in $U_i$ for $i>0$, it consequently follows that, conditional on $U_0,$ the indicator variables $I\{B_1 \}, \ldots,  I\{B_{n-1}\}$ are all increasing functions of the independent random variables $U_1, \ldots, U_{n-1}, U_{i,j} , i < n, j \geq 1.$ Consequently, given $U_0,$ the indicators $I\{B_1 \}, \ldots I\{B_{n-1}\}$ are associated, implying that 
	\begin{align*}
	    \Prob(B |U_0) &\geq \prod_{i=1}^{n-1} \Prob(B_i|U_0)\\
	    \Rightarrow \Prob(B|U_0) &\geq (\Prob(B_1|U_0))^{n-1}
	\end{align*}
	by symmetry. Taking expectations gives
	\begin{align*}
		\Prob(B) \geq & \E{ (\Prob(B_1|P_0))^{n-1}} \\
		\geq & (\E{ \Prob(B_1|P_0)})^{n-1} \\
		=& (\Prob(B_1))^{n-1}
	\end{align*}
	where the last inequality follows from Jensen's inequality. 
	\end{proof}
	To obtain an upper bound on $\Prob(B),$ let $B^*$ be the event that the best arm beats the best of arms in $\{1,\dots,n-1\}.$ That is, that arm $0$ beats the one having mean $\max_{1\leq i} P_i.$ Because $B \subset B^*$ we know
	$\Prob(B) \leq \Prob(B^*).$ Before getting into more detail on how to compute this upper bound we should note the following.
	\begin{remark}\label{rem:Pc=PB}
	It is possible for the best arm to be chosen even if it does not win all its games. Indeed, this will happen if the best arm loses to an arm that at an earlier time was eliminated, However, it is intuitive that this event has a very small probability of occurrence. Consequently, $\Prob(C) \approx \Prob(B).$
	\end{remark}
	For computing $\Prob(B_1)$ and $\Prob(B^*)$ we can use simulation with a conditional expectation estimator. Then we use this estimation to choose a proper $k$ for a given $\al$. First some preliminaries on the Gambler's Ruin problem. If arms with known mean $x$ and $y$ play a game that ends when one has $k$ more wins than the other, then the probability that the first one wins is the probability that a gambler, starting with fortune $k$ reaches a fortune of $2k$ before $0$. This gambler wins each game with probability 
	\[p = \frac{x(1-y)}{x(1-y) + y(1-x)}\]
	and if we set $r \equiv \frac{1-p}{p} = \frac{ y(1-x)}{x(1-y) } $ we have
	$$\Prob(\mbox{The arm with mean }x\;\mbox{wins}) = \frac{ 1- r^k }{ 1- r^{2k} } = \frac{ 1} { 1 + r^k } $$
	Also, using known results from the gambler's ruin problem along with Wald's equation (to account for the fact that not every round leads to a gain or a loss) it follows that the mean number of plays before stopping is
	\begin{equation}   \label{eq1}
	\E{\mbox{number of plays}} = \frac{ k( 1-r^k )}{ (r^k + 1)(x-y)}
	\end{equation}
	We can approximate $\E{N}$ by mean number of plays in the corresponding games using \cref{eq1}. Now we can prove the following:
	\begin{prop}\label{prop:vt}
	Let 
	$U$ and $V$ be independent uniform $(0, 1)$ random variables, and let $X = F^{-1}( U^{1/n} )$ and
	\begin{align*}
	    Y&=F^{-1}( U^{1/n} V),\quad & W&=F^{-1}( U^{1/n} V^{1/(n-1)})\\
	    R&=\frac{ Y( 1 - X)}{X(1-Y)}, \quad & S&=\frac{W(1 - X)}{X(1-W)}
	\end{align*}
	with $C$ and $B$ as previously defined, then $ \Prob(B) \leq \E{ \frac{ 1} { 1 + S^k } } $ and 
	$$ \Prob(C) \geq \Prob(B) \geq  (\E{ \frac{ 1} { 1 + R^k } } )^{n-1}$$ 
	$\E{N} \approx A$ where
	\begin{align*}
	 A \equiv (n-1) \E{\frac{ k( 1-R^k )}{ (R^k + 1)(X-Y)} } + \E{\frac{ k( 1-S^k )}{ (S^k + 1)(X- W)} } 
	\end{align*}
	\end{prop}
	\begin{remark}[Choosing $k$ for VT]\label{rem:k-vt}
	    Combining \cref{rem:Pc=PB} and Proposition \ref{prop:vt} we get an interval which includes $\Prob(C)$ for a given $k$ with a high probability. Now we can do an independent simulation using only uniform (0,1) random variables to find the smallest $k$ such that the interval includes the desired $\alpha$. We show an example of this in \cref{sec:exps}.
	\end{remark}
	
	\subsection{Early Elimination}\label{sec:VT-EE}
	In an improved modification of VT rule we use early elimination (EE) where if an arm is $j$ wins behind after the first $j$ rounds (that is, if the arm had all failures in the first $j$ rounds while another arm had all successes) then that arm is eliminated. To see by how much this can reduce the accuracy of VT, let us compute $\Prob(L),$ where $L$ is the event that the best arm is eliminated early. Let $0$ be the best arm, and let $1, \ldots, n-1$ be the other arms in random order like before. 
	We know 
	\[(P_0, \ldots, P_{n-1}) =_{st} (W, WU_1, \ldots, W U_{n-1} ) \] where $W = U^{1/n}.$  Consequently, letting \[(P_0, \ldots, P_{n-1}) = (W, WU_1, \ldots, W U_{n-1} ) \] yields
	\[\Prob(s_i(j) = j|W) = \E{ (U_i W)^j| W} = \frac{ W^j}{j+1} \]
	\[\Prob(L|W) = (1 - W)^j ( 1 - ( 1 -    \frac{ W^j}{j+1}     )^{n-1}  )\]
	Taking expectations gives 
	\[\Prob(L) = \E{ (1 - U^{1/n})^j ( 1 -  ( 1 - \frac{U^{j/n}}{j+1})^{n-1} ) } \]
	Let us now consider the expected number of non best arms that are eliminated early. We know $s_0, \ldots, s_{n-1}$ are conditionally independent given $W$, then
	\begin{align*}
		& \Prob(s_{n-1} = 0, & \max_{i\in\{0,\dots,n-2\}}s_i = j |W)=\\ &\Prob(s_{n-1} = 0| W ) &\Prob( \max_{i\in\{0,\dots,n-2\}}s_i = j |W)=\\
		&\Prob(s_{n-1} = 0| W ) &( 1 - \prod_{i=0}^{n-2} \Prob(s_i < j |W) )
	\end{align*}
	and 
	\begin{align*}
		\Prob(s_{n-1} = 0| W )&=  \E{ (1 - U_{n-1} W)^j|W}\\
		&=\int_0^1 (1 - xW)^j \, dx
		\\&=  \frac{1}{W} \int_{1-W}^1    y^j \, dy=\frac{1 - (1-W)^{j+1} }{ (j+1) W} 
	\end{align*}
	and
	\begin{align*}
		\prod_{i=0}^{n-2} \Prob(s_i < j |W) &=  (1 - W^j)&  (\Prob(s_1 < j |W))^{n-2} \\
		&=  (1 - W^j) & (\E{ 1- (U_1W)^j |W})^{n-2} \\
		&=  (1 - W^j) & ( 1 - \frac{W^j}{j+1} )^{n-2}
	\end{align*}
	Hence, with $N^*$ being the number of non-best arms that are eliminated early, we have $\E{N^*} = (n-1) D$ where $D$ is
	\begin{multline*}
	    \E{\frac{1 - (1-W)^{j+1} }{ (j+1) W} \left( 1 -   (1 - W^j)  \left( 1 - \frac{W^j}{j+1} \right)^{n-2}  \right)}
	\end{multline*}
	which could be computed via simulation. 
	
	We can also use a randomized $k$ as follows. Let $P_k(C)$ be the probability of a correct choice when using VT with critical value $k,$ and suppose $P_{k-1}(C) < \al < P_{k}(C).$
	The randomized rule that chooses $k$ with probability 
	\[p=\frac{\al - P_{k-1}(C) }{ P_{k}(C)  - P_{k-1}(C) } \]
	and $k-1$ with probability $1-p$ yields the correct choice with probability $\al.$  Another possibility is to use VT along with EE parameter $j^*,$ where $j^*$ is the smallest value $j$ for which EE at $j$ results in a correct choice with probability at least $\al.$ Of course, we could also use VT with critical value $k$ and randomize between EE at $j^* - 1$ and EE at $j^*$. Example \ref{ex:VT} in \cref{sec:exps} is an instance where randomizing among VT rules results in a smaller $\E{N}$ than does VT with EE.
	\section{Play the Winner Rule, Bernoulli Case}\label{sec:PW}%
	Next simple algorithm proposed in the Bayesian setting is {\it{play the winner}} (PW) rule, which in each but the last round continues to sample from each alive (not eliminated) arm until it has a failure. Our proposed PW
	is described concisely in \cref{alg:PW}. In this algorithm $a_i$ is the indicator of alive arms and $l_i$ is the indicator of arms with a loss/failure.
	\begin{algorithm}[tb]
      \caption{Play the winner (PW)}
      \label{alg:PW}
    \begin{algorithmic}
      \STATE set $a_i=1$, and $s_i(0)=0$ for all $i\in\{1,\dots,n\}$
      \REPEAT
      \STATE set $l_i=0$ for all $\{i: a_i=1\}$
      \REPEAT
      \STATE - Sample all arms in $\{i: l_i=0\}$
      \STATE - Set $l_i=1$ for the arms with failure
      \STATE - If $\exists j$ s.t. $l_j=0$ and $l_i=1\ \forall i\neq j$, and $s_j(s)\geq \max_{i\neq j}s_i(s)+k$, then stop and return $j$
      \UNTIL{$|\{i: l_i=0\}|=1$}
      \STATE - Set $a_i=0$ if $\exists j\neq i$ s.t. $s_i(t)\leq s_j(t)-k$
      \UNTIL{$|\{i: a_i=1\}|=1$}
    \end{algorithmic}
    \end{algorithm}
	\noindent
	We should note that:\\
	1. If we define a round by saying that each alive arm is observed until it had a failure, then when $F$ is the uniform $(0, 1)$ distribution, the expected number of plays until the first arm has a failure could be infinite. Therefore we define rounds using subrounds so that the mean number of plays is finite. For instance, suppose $F$ is uniform $(0, 1)$ and  $P_{[1]} > P_{[2] } >  \ldots > P_{[n]}$. Let $N_i$ denote the number of plays in the first round of the arm with probability $P_{[i]}.$ Then, as the density of $P_{[i]}$ is
	$$f_{P_{[i]}} (p)  = \frac{ n!p^{n-i} (1-p)^{i-1}} {(i-1)! (n-i)!}dp ,\ 0<p<1$$
	it follows that with our PW algorithm $\E{N_i} \leq \E{ \frac{1}{1 - P_{[i]} } } < \infty $ when $i > 1.$ In addition,
	\[N_1  \leq k + \max_{2 \leq i \leq n} N_i \leq k + \sum_{i=2}^n N_i\]
	\[\Rightarrow \E{N_1} \leq k + \sum_{i=2}^n \E{N_i} < \infty\]
	
	2. The PW rule as defined in \cite{sw1} and \cite{sw2} was such that the arms are initially randomly ordered. In each round, the alive arms were observed in that order, with each arm being observed until it had a failure. If at any time one of the arms had $k$ fewer successes than another arm, then the former is no longer alive. The process  ends when only a single arm is alive which is declared as the best. 
	
	\subsection{Analysis of PW}\label{sec:PW-an}%
	To begin, suppose there are only $2$ arms, and that their success probabilities are $p_1 > p_2.$ Also, suppose we are going to choose an arm by using the procedure which in each round plays each arm until it has a failure. We stop at the end of a round if one of the arms, say $l$, has $s_l\geq s_{l'}+k$, then $l$ is chosen as the best. Let $q_i = 1-p_i, i=1,2,$ and let $X_{i,r} , i = 1,2,\ r \geq 1,$ be independent with \[\Prob(X_{i,r}  = j) = q_i\,p_i^j, j \geq 0.\]
	We know $X_{i,r}$ is stochastically equal to the number of successes of arm $i$ in round $r$. Letting $Y_r = X_{1,r} - X_{2,r}, r \geq 1$ we get \[\E{e^{\theta Y_r}} = \frac{q_1} {1-p_1 e^{\theta}}\frac{q_2}{1-p_2 e^{\theta}}\]
	It is now easy to check that $\E{e^{\theta Y_r}} = 1$ if   $e^{\theta} = p_2/p_1$. That is,  $ \E{ (p_2/p_1)^{Y_r}} = 1. $  If we now let  $S_m = \sum_{i=1}^m Y_i$ then
	$\; (p_2/p_1)^{S_m}, \, m  \geq 1\, $ is a martingale with mean $1$. Letting
	$\tau = \min \{m: S_m \geq k \;\;\mbox{or} \;\; S_m \leq -k \} $
	it follows by the martingale stopping theorem \cite{ro} that $\E{  (p_2/p_1)^{S_{\tau}}} = 1 $.
	Let $\,p = \Prob(S_{\tau} \geq k)$ be the probability that arm $1$ is chosen. Then
	\begin{align*}
		1 = \E{  (p_2/p_1)^{S_{\tau}}}
		= &\E{(p_2/p_1)^{S_{\tau}}|S_{\tau}\geq k}p +\\ &\E{(p_2/p_1)^{S_{\tau}} | S_{\tau} \leq -k}(1-p)
	\end{align*}
	Letting $X_i, i = 1,2,$ have same distribution as $X_{i,r},$ it follows, by  the lack of memory of $X_i$, that
	\begin{align*}
    \E{  (\frac{p_2}{p_1})^{S_{\tau}} | S_{\tau} \geq k}&= (\frac{p_2}{p_1})^k \E{(\frac{p_2}{p_1})^{X_1}} &= (\frac{p_2}{p_1})^k(\frac{q_2}{q_1})
	\\\E{  (\frac{p_2}{p_1})^{S_{\tau}} | S_{\tau} \leq - k}  &=  (\frac{p_2}{p_1})^{-k} \E{(\frac{p_2}{p_1})^{-X_2}}  &= (\frac{p_2}{p_1})^k (\frac{q_2}{q_1})
	\end{align*}
	Substituting back yields that
	\begin{equation}  \label{eq1-p}
	p = \frac{   1 -   (p_1/p_2)^k (q_2/q_1)  }  {   (p_2/p_1)^k(q_1/q_2)  - (p_1/p_2)^k (q_2/q_1) }
	\end{equation}
	Conditioning on which arm wins yields that \[\E{S_{\tau}} =  (k + \E{X_1} )p + (-k - \E{X_2})(1-p)\]
	Letting $\, m_i = \E{X_i} =  1/q_i - 1 = p_i/q_i,$ the preceding gives $\E{S_{\tau}} = p( m_1 + m_2  +  2k) - (m_2 + k)$. Now Wald's equation yields
	\begin{align}\label{eb}
	\E{{\tau}}
	=\frac{  p( m_1 + m_2  +  2k) - m_2 - k  } {m_1 - m_2}
	\end{align}
	Because $X_{1,r} + X_{2,r} + 2 $ is the  number of plays in round $r$,  it follows that the total number of plays in this setting, call it $T$, is  $\sum_{r=1}^{\tau} (X_{1,r} + X_{2,r} + 2).$  Applying Wald's equation and using \cref{eb}  gives 
	\[\E{T} =  ( p( m_1 + m_2+ 2k) - m_2 - k  ) \frac{ m_1 + m_2 + 2} {m_1 - m_2}\]
	Now suppose we want to calculate the probability of choosing the best arm correctly along with the number of samples. Let $B(p_1, p_2)$ and $N(p_1, p_2)$ be, respectively, the probability that the arm with value $p_1$ is chosen and the mean number of plays before stopping. From \cref{eq1-p}
	we have
	\begin{equation}
	B(p_1, p_2) = \frac{   1 -   (p_1/p_2)^k (q_2/q_1)  }  {   (p_2/p_1)^k(q_1/q_2)  - (p_1/p_2)^k (q_2/q_1) }
	\end{equation}
	Because PW would stop play once the winning arm is ahead by $k$, whereas $\E{T}$ is the mean number of plays when we continue on until a failure occurs, by conditioning on which arm wins we obtain that
	\begin{align*}
	    N(p_1, p_2)=&   \E{T} - \frac{ B(p_1, p_2)}{q_1}  - \frac{1 -  B(p_1, p_2) }{q_2 }    \\
	=&   \big( B(p_1, p_2) ( m_1 + m_2  +  2k) - m_2 - k \big)\\
	&\frac{m_1 + m_2 + 2}        {m_1 - m_2}  - \frac{ B(p_1, p_2)}{q_1}  - \frac{1 -  B(p_1, p_2) }{q_2 }  
	\end{align*}
	where $m_i = p_i/q_i$. 
	
	Now suppose there are $n$ arms 
	with prior distribution $F$. Akin to our VT analysis, suppose that all arms participate in each round and let arm $0$ be the best arm, and randomly number the other arms as $1, \ldots, n-1.$ Then we have

	\begin{lemma}\label{lemma:PW-Bs}
	For the PW algorithm, with $B \equiv B_1B_2 \cdots B_{n-1} $ we have \[\Prob(C)\geq \Prob(B)\geq  (\Prob(B_1))^{n-1}\]
	Also, $\,\Prob(B) \leq \Prob(B^*),\, $ where  $B^*$ is the event that arm $0$ beats the best of arms $1, \ldots, n-1.$
	\end{lemma}
	
	Our preceding  analysis yields the following corollary.
	
	\begin{corr}
	With
	$U$ and $V$ being independent uniform $(0, 1)$ random variables, we define
	\begin{align*}
	    X &= F^{-1}( U^{1/n} ),\\
	    Y &= F^{-1}( U^{1/n} V)\\
	    W &= F^{-1}( U^{1/n} V^{1/(n-1)})
	  \end{align*}  
	    then
	    \begin{align*}
	    \Prob(B_1)&=  \E{B(X, Y)}\\
	    \Prob(B^*) &= \E{B(X, W)}
	\end{align*}
and
	\[N \approx A = (n-1) \E{N(X,Y)} + \E{N(X,W)} \]
	\end{corr} 
	\begin{remark}[Determine $k$ for PW]
	    Based on \cref{lemma:PW-Bs} and a similar argument to \cref{rem:Pc=PB} we can derive a simple procedure like in \cref{rem:k-vt} to find appropriate $k$ for PW algorithm.
	\end{remark}
	
	We also analysed the early elimination version of PW, PW-EE. Since there is no obvious improvement using PW-EE based on our numerical studies, we defer this to the appendix. Also \cref{sec:exps} includes a numerical comparison of VT and PW.
	\section{The Normal Case}\label{sec:normal}%
	In this section we suppose that rewards of arm $i$ are independent Normal random variables with mean $\mu_i$ and variance $\si^2$ where $\si^2$ is known and $\mu_1, \ldots, \mu_n$ are the unknown values of $n$ independent standard Normal random variables, i.e. $F$ is the standard Normal. 
	The VT rule with parameter $c \in \mathbb{R}^+$ for this case is given in \cref{alg:VT-normal}.
	\begin{algorithm}[tb]
  \caption{VT for Gaussian rewards}
  \label{alg:VT-normal}
\begin{algorithmic}
  \STATE \textbf{Input}: all the arms.
  \REPEAT
  \STATE - Sample each remaining arm once.
  \STATE - Eliminate any arm $i$, if $\exists j\neq i$ s.t. $S_i(k) < S_j(k) -c$.
  \UNTIL{there is one arm remaining arm}
\end{algorithmic}
\end{algorithm}

Before elaborating on how to choose a proper $c$ we present some preliminaries concerning Normal partial sums. 

\subsection{Preliminaries}\label{sec:prelim}%
Let  $\Phi$ be the standard Normal distribution function. Define $R(a)=\frac{\Phi(a)}{1-\Phi(a)}$, then we have the following Lemma for a Normal random variable.
\begin{lemma}\label{lemma:Eprelim}
If $Z$ is a Normal random variable with mean $\mu$ and variance $1$, then
\begin{align*}
	\E{e^{- 2 \mu Z}|Z > 0}  &=  R(-\mu)   \\
	\E{Z|Z > 0}  &=  \mu + { e^{ -\mu^2/2} }/( {\sqrt{ 2 \pi } \Phi(\mu)})  \\
	\E{e^{- 2 \mu Z}|Z < 0}   &= R(\mu) \\
	\E{Z|Z < 0}  &=  \mu - {  e^{ -\mu^2/2} }/( {\sqrt{ 2 \pi } (1 - \Phi(\mu) }) 
\end{align*}
\end{lemma}
We prove this Lemma in in the appendix. We also prove the following Lemma for the partial sum and its stopping time in the appendix.
\begin{lemma}\label{lemma:ESn}
	Let $S_m = \sum_{i=1}^m Z_i,\; m \geq 1,$ where $Z_i, i \geq 1 $ are independent Normal random variables with mean $\mu > 0$ and variance $1.$ For  given $b > 0,$ let 
	\[\tau = \min\{m: \; S_m < - b \;\mbox{or}\; S_m > b\}.\] 
	then
	\begin{align*}
	      R(-\mu) e^{- 2 \mu b} &<& \E{e^{- 2 \mu S_{\tau}}| S_{\tau} > b} &<& e^{- 2 \mu b}\\
	      e^{2 \mu b} &<&  \E{e^{- 2 \mu S_{\tau}}| S_{\tau} < - b} &<&
	      e^{2 \mu b}R(\mu)            
	\end{align*}
\end{lemma}
Now using these Lemmas we can derive a tail bound for the partial sum in the following proposition.
	
	\begin{prop}\label{prop:P_Snb}
	Let's have same definitions as in \cref{lemma:ESn}, then
	\begin{equation*}      
	\frac{ e^{2 \mu b}     - 1}  { e^{2 \mu b}    -   R(-\mu)   e^{- 2 \mu b}  }          <    \Prob(S_{\tau} > b) <    \frac{     e^{2 \mu b}   R(\mu)   - 1}  {    e^{2 \mu b}    R(\mu)  -   e^{-2 \mu b}   }      
	\end{equation*}
	\end{prop}
	
	\begin{proof}
	Let $p  = \Prob(S_{\tau} > b).$ Because $\E{e^{-2 \mu Z_i}} = 1,$ it follows that  $\{e^{- 2 \mu S_m}\}_{n \geq 1}$ is a martingale with mean $1$.  Hence, by the martingale stopping theorem
	\begin{align*}
		1&=\E{e^{- 2 \mu S_{\tau}}} \\
		&=\E{e^{- 2 \mu S_{\tau}}| S_{\tau} > b} p + \E{e^{- 2 \mu S_{\tau}}| S_{\tau} < - b}(1-p) 
	\end{align*}
	then rearranging gives
	\begin{equation}  \label{p}
	p =     \frac{    \E{e^{- 2 \mu S_{\tau}}| S_{\tau} < - b} - 1}{   \E{e^{- 2 \mu S_{\tau}}| S_{\tau} < - b} -  \E{e^{- 2 \mu S_{\tau}}| S_{\tau} > b} } 
	\end{equation}
	Since $\;\frac{ x - 1}{x-y}$ for $0<y<1<x$ increases in both $x$ and $y$, the desired inequalities follow from \cref{lemma:ESn}.
	\end{proof}
	
	We can also bound the expectation of $\tau$ as follows. 
	
	\begin{prop}
	With $\tau$ as previously defined in \cref{lemma:ESn} we have
	\begin{align}  
	\E{{\tau}}  \leq  & \frac{ e^{2 \mu b}    R(\mu) - 1}  {    e^{2 \mu b}    R(\mu)   -   e^{-2 \mu b}   }   \bigg( \frac{2b}{\mu} + \frac{ e^{- \mu^2/2} } {\mu\sqrt{ 2 \pi } \Phi( \mu) } + 1\bigg) - \frac{b}{\mu}  \label{bbb}\\
	\E{{\tau}}  \geq  &     \frac{ e^{2 \mu b}     - 1}  { e^{2 \mu b}    -   R(-\mu)   e^{- 2 \mu b}  }   \,  \bigg(\frac{2b}{\mu} +  \Psi(\mu) \bigg)  - \frac{b}{\mu} + \Psi(\mu) \label{bbb2}
	\end{align}
	where $\Psi(\mu)=1-\frac{  e^{ -\mu^2/2} }{\mu \sqrt{ 2 \pi } (1 - \Phi(\mu) )}$.
	\end{prop}
	\begin{proof}
	 Wald's equation gives
	\begin{align*}
		\E{{\tau}} \mu &= \E{S_{\tau}|S_{\tau} > b} p + \E{S_{\tau} |S_{\tau} < - b} (1-p) \\
		& \leq  (b + \E{W|W > 0} ) p -b (1-p) \\
		&= p(2b + \frac{ e^{- \mu^2/2} } { \sqrt{ 2 \pi } \Phi( \mu) }  + \mu) - b 
	\end{align*}
	where the first inequality used, as shown in \cref{lemma:ESn}, that $S_{\tau} | \{S_{\tau} > b \}$ is stochastically smaller than $ b + W| \{W> 0\}$. Inequality (\ref{bbb}) now follows from Proposition \ref{prop:P_Snb}.  The lower bound follows from writing $\E{{\tau}}$ as follows
	\begin{align*}
		 &\E{S_{\tau}|S_{\tau} > b} p + \E{S_{\tau} |S_{\tau} < - b} (1-p) &\geq&\\
		 &b p  + (-b + \E{W_1|W_1 < 0} )(1-p) &=&\\&
		\bigg(2b - \mu + \frac{  e^{ -\mu^2/2} } {\sqrt{ 2 \pi } (1 - \Phi(\mu) } \bigg) p \ \ -b \ + \ \ \mu &- &\frac{  e^{ -\mu^2/2} } {\sqrt{ 2 \pi } (1 - \Phi(\mu) } 
	\end{align*}
	where the inequality used  that $S_{\tau} | \{S_{\tau} < - b \}$ is stochastically larger than $ -b + W_i | \{W_i< 0\}.$  Now inequality (\ref{bbb2}) follows from  Proposition  \ref{prop:P_Snb}.
	\end{proof}
	
	Now we can approximate $p=\Prob(S_{\tau} > b)$  and $\E{{\tau}}$ by ``neglecting the excess" and assuming  $(S_{\tau}|S_{\tau} > b) \approx_{st}  b$ and   $(S_{\tau}|S_{\tau} < -b) \approx_{st}  -b.$  From \cref{p} this  gives that 
	\begin{equation}  \label{pp}
	p \approx\frac{  e^{ 2 \mu b} -1}{ e^{ 2 \mu b} - e^{ -2 \mu b}}
	\end{equation}
	Also,
	$\; \mu \E{{\tau}} \approx b p - b(1-p) ,\;$ and so \cref{pp} gives that 
	\begin{equation}  \label{ex}
	\E{{\tau}} \approx  \frac{  2b( e^{ 2 \mu b} -1)}{\mu( e^{ 2 \mu b} - e^{ -2 \mu b})} - {b}/{\mu}
	\end{equation}
	
	Example \ref{ex:normal1} gives an instance of this approximation that validates it.
	\begin{ex}\label{ex:normal1}
	Suppose $b = 3, \mu = 1,$ then \cref{bbb},   \cref{bbb2},  and \cref{ex}  yield that 
	$2.9838     \leq \E{{\tau}} \leq  4.2842$   , and $\quad \E{{\tau}} \approx 2.9852    $
	\end{ex}
	It is easy to generalize the above results to the case where variance is not 1 as follows.
	\begin{corr}\label{corr:P_SNsigmaC}
	Let $S_m = \sum_{i=1}^n V_i,\; n \geq 1,$ where $V_i, i \geq 1 $ are independent Normal random variables with mean $\mu > 0$ and variance $2 \si^2.$  For  given $c > 0,$ let $\tau$ be as before, then
    \begin{align*}
         \frac{ e^{ \mu c/\si^2}     - 1}  { e^{ \mu c/\si^2}    -  R( - \frac{ \mu } { \si \sqrt{2}} )   e^{-  \mu c/\si^2}  }          <    \Prob(S_{\tau } > c) \\    
         \frac{     e^{ \mu c/\si^2}    R(  \frac{ \mu } { \si \sqrt{2}} )   - 1}  {    e^{ \mu c/\si^2}     R(  \frac{ \mu } { \si \sqrt{2}} )    -   e^{- \mu c/\si^2}   } > \Prob(S_{\tau } > c)      
    \end{align*}
	Moreover,
	$$ \E{\tau } \approx   \frac{  2c \, ( e^{ \mu c/\si^2} -1)}{\mu( e^{ \mu c/\si^2} - e^{ - \mu c/\si^2})}-  {c}/{\mu}   $$
	\end{corr}
	\begin{proof}
	  Let $Z_i = \frac{V_i}{\si \sqrt{2}},$ note that $\E{Z_i} = \frac{ \mu}{\si  \sqrt{2}}.$  Now, using $b = \frac{ c}{\si  \sqrt{2}},$ apply Proposition \ref{prop:P_Snb} and \cref{ex} to get the desired result.
	\end{proof}
	
	\subsection{Analyzing the VT Rule in the Normal Case}\label{VT-an-Norm}
	In this section, we derive a lower bound and an effective approximation of $\Prob(C),$ by a similar argument as in the Bernoulli case. With similar indexing of arms we imagine a ``Gambler's Ruin" game between arms $0$ and $i$ where the goal is $c$. We can again show exactly as before that 
	\[\Prob(C) \geq \Prob(B) \geq (\Prob(B_1))^{n-1}\] and $ \Prob(B) \leq \Prob(B^*) $. Given the mean values $\mu_0, \mu_1, \ldots, \mu_{n-1}$, the difference between a sample from arm $0$ and another arm is a Normal random variable with variance $2 \si^2.$  Letting $LB(\mu)$ and $UB(\mu)$
	be the lower and upper bounds on $\Prob(S_{\tau} > c)$ in Corollary \ref{corr:P_SNsigmaC}, it yields the following proposition. 
	\begin{prop}
	Let 
	$U$ and $V$ be independent uniform $(0, 1)$ random variables, and let
	\begin{align*}
	    X =& \Phi^{-1}( U^{1/n} ) -  \Phi^{-1}( U^{1/n} V) 
	    \\Y =& \Phi^{-1}( U^{1/n} ) -  \Phi^{-1}( U^{1/n} V^{1/(n-1) } )
	\end{align*}
	Then
	\[ \Prob(C) \geq \Prob(B) \geq  (\E{ LB(X)} )^{n-1}\] and \[\Prob(B) \leq \E{UB(Y)}\]
	\end{prop}
	Now we can approximate $\E{N}$ 
	by estimating the mean number of plays of each non best arm by the mean number of plays in their game against the best arm, and also approximating the mean number of plays of the best arm by the mean number of plays in its game against the second best arm. Hence, using \cref{ex} we have
	\[\E{N} \approx A \equiv        (n-1)  \E{M(X)} + \E{M(Y)}\]
	where 
	\[M(\mu) =  \frac{  2c \, ( e^{ \mu c/\si^2} -1)} {\mu( e^{ \mu c/\si^2} - e^{ - \mu c/\si^2})} -  \frac{c}{\mu}\]

\section{Experimental Results}\label{sec:exps}

In this section we present experiments that establish the efficiency of our algorithms and help us evaluate our approximations of $N$ and $\Prob(C)$.

\subsection{VT, Bernoulli}
	 Here we assume $F$ is the uniform $(0, 1)$ distribution and arm rewards are Bernoulli, then we compare the estimations with the simulation results. Let $\tilde{\sigma}$ be the standard deviation of estimate of $A$ and $m$ be the number of simulation runs. \cref{tab:exper} shows the results of the algorithm and \cref{tab:estimate} shows the estimation for the same cases with $m=10^6$.
	For instance, \cref{tab:estimate} shows that to obtain $99$ percent accuracy with $n=10$ arms, it seems that the elimination number has to be $57$ or $58$. Also by comparing \cref{tab:exper} and \cref{tab:estimate} we can confirm that the quality of the estimations is very high.
    \begin{table}[ht]
        \centering
        \begin{tabular}{c c c c c c}
        \hline
        $n$& $k$& $\Prob(C)$& $\E{N}$& $\tilde{\sigma}$\\\hline
        10& 50& 0.9886 & 5466.318& 17.34\\
        5& 10& 0.9540& 358.3993& 1.156\\\hline
        \end{tabular}
        \vspace{.05in}	
        \caption{VT Experimental results.}
        \label{tab:exper}
    \end{table}
    \begin{table}[ht]
        \centering
        \begin{tabular}{c c c c c c c}
        \hline
             $n$& $k$& $\Prob(B_1)$& $\Prob(B^*)$& $A$& $\tilde{\sigma}$  \\\hline
             10& 50& 0.9885& 0.9889& 5460.539& 26.6991\\
             & 57 & 0.9896 & 0.9901& 6462.372& 32.9701\\
             & 58& 0.9900& 0.9903& 6545.46&  32.4571\\
             5&  10& 0.9523& 0.9561& 358.3983& 0.8256\\
             \hline
        \end{tabular}
        \vspace{.05in}	
        \caption{VT Estimation results.}
        \label{tab:estimate}
    \end{table}
	
	We compare the VT rule with recent algorithms in the literature.
	We use \emph{Track and Stop} (TaS) algorithm from \cite{gk}. We let TaSC stand for TaS with C tracking and TaSD for D tracking. We also employ \emph{Chernoff Racing} (ChR) \cite{gk}, \emph{Kullback-Leibler Racing} (KL-R), and \emph{KL-LUCB} \cite{kaufmann2013information} algorithms. 
	\cref{tab:other} is borrowed from \cite{gk}. The table consists of the results for two cases: the first case having $n=4$ with probabilities $(0.5, 0.45, 0.43, 0.4), $ and  the second having $n=5$ with probabilities $(0.3, 0.21, 0.20, 0.19, 0.18)$.  The parameters of these algorithms are chosen to guarantee at least $\al=0.9$. \cref{tab:other-VT} shows the result of $m=10^4$ simulation runs of VT algorithm for these cases.
	\begin{table}[ht]
		\begin{center}
			\begin{tabular}{ c c c c c c}
				\hline
				Case  & TaSC & TaSD  & ChR & KL-LUCB & KL-R\\
				\hline
				1 & 3968 & 4052 & 4516 & 8437 & 9590\\
				2 & 1370& 1406 & 3078 & 2716 & 3334 \\\hline
			\end{tabular}
		\end{center}
		\vspace{.05in}	
		\caption{Recent algorithms results}
		\label{tab:other}
	\end{table}
	\begin{table}[ht]
		\begin{center}
			\begin{tabular}{ c c c c c c c c}
				\hline
				Case  & $\al$ &$k$ &   $\Prob(C)$& $\E{N}$\\
				\hline
				1& 0.99& 42& 0.9999& 2738\\
				& 0.97& 15& 0.9998& 905\\
				2& 0.99 & 47& 0.9999& 2372\\
				& 0.97& 16 & 0.998& 832
				\\\hline
			\end{tabular}
		\end{center}
		\vspace{.05in}	
		\caption{VT results to compare with table \cref{tab:other}}
		\label{tab:other-VT}
	\end{table}
	
	Because our algorithm assumes knowledge of a prior distribution, in cases where there is no reason to assume that we know what the prior is, it seems reasonable to assume a uniform $(0, 1)$ prior and choose a larger accuracy than is actually desired. So suppose we do so and require $\al=0.99$ and $\al=0.97$. \cref{tab:other-VT} shows the results with proper $k$'s for VT.
	As we can see, the VT algorithm significantly outperforms the newer algorithms even with fixed probabilities. Although, to be fair we should mention that, under the uniform $(0, 1)$ prior, $k=5$ is sufficient for both $n=4$ or $n=5$ to obtain $90$ percent accuracy based on Proposition \ref{prop:vt} estimations. In case 1, this yields $\E{N}=166$, but $\Prob(C)=0.601$. In case 2 it has $\E{N}=213.2$, with $\Prob(C)=0.81.$

\subsection{VT with EE, Bernoulli}
	First we evaluate the estimations of $\Prob(L)$ and $\E{N^*}$ to illustrate how efficient EE could be. \cref{tab:PL} shows the values of $\Prob(L)$ and $\E{N^*}$ for a variety of $n$ and $j$ when $F$ is the uniform $(0, 1)$ distribution. We can observe how small $\Prob(L)$ and $\E{N^*}$ are.
	\begin{table}[ht]
		\begin{center}
			\begin{tabular}{c c c c}
			    \hline
				n & j & $\Prob(L)$ & $\E{N^*}$ \\\hline
				5 & 2 & 0.02053 & 1.317 \\
				& 3 & 0.00336 & 0.851  \\
				& 4 & 0.00059 & 0.590  \\
				& 5 & 0.00011 & 0.432  \\\hline
				10 & 2 & 0.01278 & 3.234 \\
				& 3 & 0.00201 & 2.310 \\
				& 4 & 0.00033 & 1.731 \\
				& 5 & 0.00006 & 1.343 \\\hline
				20 & 2 & 0.00429 & 6.659 \\
				& 3 & 0.00053& 4.978  \\
				& 4 & 0.00007& 3.942  \\
				& 5 & 0.00001& 3.229 \\\hline
			\end{tabular}
			\vspace{.05in}
			\caption{Numerical Example of VT with early elimination}
			\label{tab:PL}
		\end{center}
	\end{table}

Next example compares VT with VT-EE to show its effectiveness. 
\begin{ex}\label{ex:VT}
Suppose $n= 5$  and $\al =0.95.$ 
\cref{tab:VT-VT-EE} shows the simulated results based on $5\times10^5$ runs. Here $\tilde{\sigma}$ is standard deviation of the $\E{N}$ estimator.)

\begin{table}[ht]
	\begin{center}
		\begin{tabular}{c c c c c c}
			\hline
			Algorithm & $k$ & $j$ & $\Prob(C)$ & $\E{N}$ & $\tilde{\sigma}$ \\
			\hline
			VT & 9   &- & 0.948 & 313.64 & 2.11 \\
			& 10 &- & 0.954 &358.40 &1.156  \\
			\hline
			VT-EE & 10  & 2 &   0.9385 & 335.52 & 8.29 \\
			& 10  & 3 &  0.9523  &348.27 & 2.70\\
			\hline
		\end{tabular}
	\end{center}
	
	\vspace{.05in}	
	\caption{Results for VT and VT with early elimination, VT-EE}
	\label{tab:VT-VT-EE}
\end{table}
Based on these results, randomizing among VT with $k=9$ and $k=10$  to obtain $\Prob(C) =0.95$  has  mean $(2/3)313.64 + (1/3)358.40= 328.56,$ which is smaller than what can be obtained with VT with EE. It is also better than the recently proposed algorithms. Of these, the Chernoff-Racing bound algorithm performs the best between others in the literature, giving an average number of $423.4$ with accuracy $0.953.$
\end{ex}

	\subsection{VT versus PW, Bernoulli}\label{sec:VT-PW}
	Based on  numerical experiments, VT and PW have    roughly similar performances  when $F(x) = x$.
	When $n=5,$ simulation yields the following results for PW in \cref{tab:PW}.
	\begin{table}[ht]
		\begin{center}
			\begin{tabular}{c c c c }
				\hline
				$k$ & $\Prob(C)$ & $\E{N}$ & $\tilde{\sigma}$ \\
				\hline
				42 &  0.9494 &319.78 & 1.64 \\
				43 &  0.9502 & 327.80 &  1.65  \\
				48 &  0.9543 & 375.4 & 0.899 \\  
				\hline 
			\end{tabular}
		\end{center}
		\vspace{.05in}	
		\caption{Numerical examples of PW}
		\label{tab:PW}
	\end{table}
	The results show that choosing PW with $k=42$ with probability $0.25$  and $k=43$  with probability $0.75$  results in $\Prob(C) =0.95,$ and requires,  on average,   $325.795$ observations, which is slightly less than the average of  $328.56$  which, as shown in Example \ref{ex:VT},  can be obtained by
	a randomization of  VT rules to obtain $\Prob(C) =0.95.$  On the other hand if we wanted  $\al =0.954,$ then  both VT with $k=10$  and PW with $k=48$ achieve that, with  VT having a mean of $358.4$ observations,  compared to $375.4$ for PW. Because the average number of trials needed for PW with $k=47$ is $367.05,$ randomizing between PW(47) and PW(48) still would not be as good as VT(10).

\subsection{Experimental Result for Normal rewards}%
This section includes experiments that aim at comparing our algorithms with the literature for Normal rewards with standard Normal prior.
In \cref{Fig:PcNormal} and \cref{Fig:EnNormal} we use simulation to compare the performance of VT with the most quoted algorithms of the recent literature: \emph{lilUCB}, \emph{TrackAndStop}, and \emph{Chernoff-Racing}.
The lilUCB algorithm \cite{jmnb} uses upper confidence bounds that are based on the law of the iterated logarithm  for the expected reward of the arms.  At each stage it uses  the arm with the largest upper bound. We use a heuristic variation of the lilUCB, called lilUCB-H, which performs somewhat better than the original \cite{jmnb}. The TrackAndStop algorithm in \cite{gk} tracks the lower bounds on the optimal proportions of the arm rewards and uses a stopping rule based on Chernoff's Generalized Likelihood Ratio statistic. The Chernoff algorithm is similar to TrackAndStop, but rather than track the optimal proportions it instead chooses between the empirical best and  second-best. 
The results below are based on $10^4$ simulation runs, with each run beginning by resampling $\mu_1, \ldots, \mu_n$ from $F$ (randomized). \cref{tab:VTk} shows the proper $k$ for each problem instance. As the results show, recent algorithms are over conservative, i.e. they stop very late and their $\Prob(C)$ is extra large which makes their $\E{N}$ too large. This is aligned with the observation in \cite{russo2018simple} which states that many algorithms in the literature have conservative stopping rule and could be improved a lot.

\begin{figure}[ht]
\centering
  \centering
  \includegraphics[width=.5\textwidth]{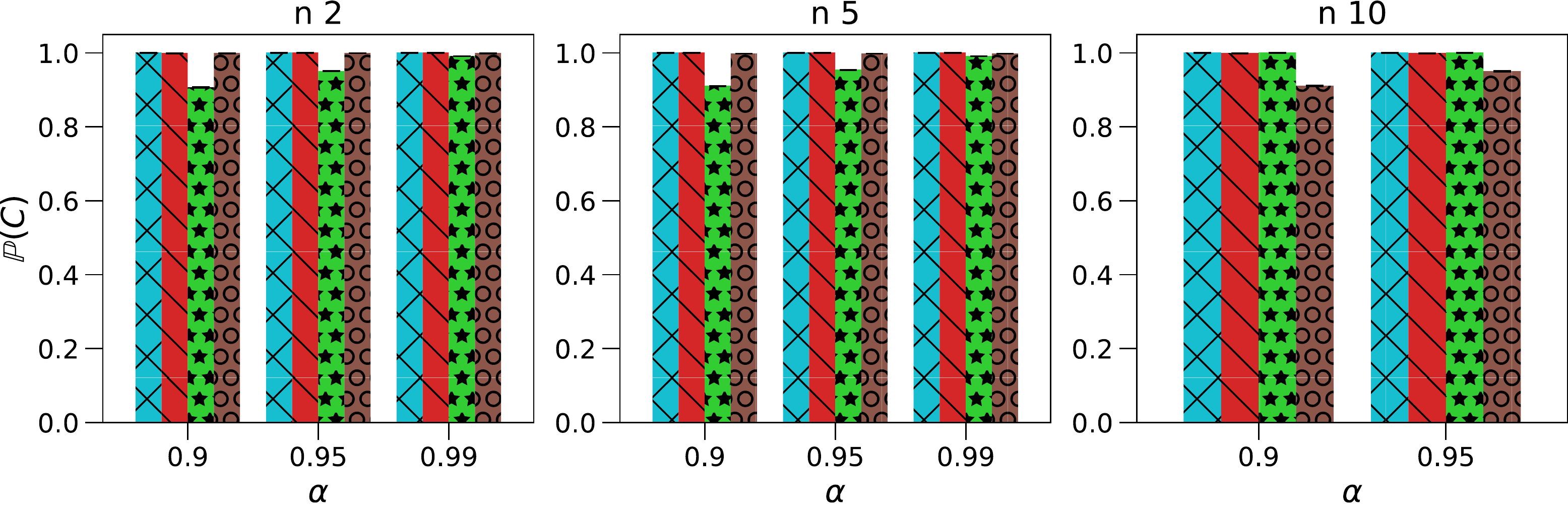}
    \caption{$\Prob(C)$ for Normal rewards.}\label{Fig:PcNormal}
\end{figure}

\begin{figure}[ht]
\centering
  \centering
  \includegraphics[width=.5\textwidth]{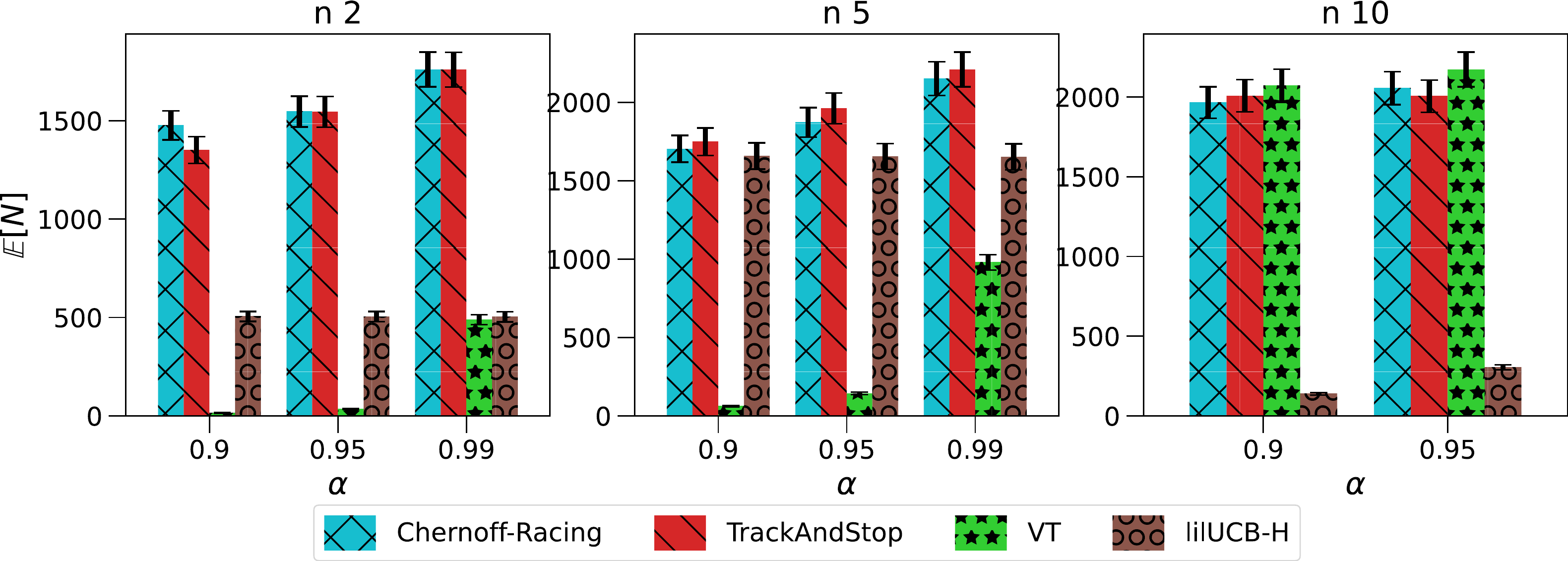}
    \caption{$\E{N}$ for Normal rewards.}\label{Fig:EnNormal}
\end{figure}

\begin{table}[ht]
	\begin{center}
		\begin{tabular}{ccccccccc}
			\hline
			$n$ & 2 &&& 5 &&& 10& \\\hline
			$\alpha$ & 0.9 & 0.95 &0.99& 0.9 & 0.95 &0.99& 0.9 & 0.95\\
			$k$ &5 & 11&40 &7.8 & 16.3 &85 &10.3&20.6 \\   
			\hline 
		\end{tabular}
	\end{center}
	\caption{$k$ for VT with Normal rewards.}
	\label{tab:VTk}
\end{table}

\section{Conclusion}
We study the problem of best arm identification in multi-armed bandit for fixed confidence setting under the Bayesian setting with both Bernoulli and Normal rewards. We use the classical \textit{vector at a time} and \textit{play the winner} algorithms and analyze them in a novel way using the Gambler's ruin problem and martingales stopping theorem. We also derive easy estimations for these algorithms. Numerical experiments show that these rules compare favorably with recently proposed algorithms in terms of having higher accuracy and smaller sample complexity.

\bibliographystyle{IEEEtran}

\clearpage
\bibliography{refs}

\begin{thebibliography}{10}
\providecommand{\url}[1]{#1}
\csname url@samestyle\endcsname
\providecommand{\newblock}{\relax}
\providecommand{\bibinfo}[2]{#2}
\providecommand{\BIBentrySTDinterwordspacing}{\spaceskip=0pt\relax}
\providecommand{\BIBentryALTinterwordstretchfactor}{4}
\providecommand{\BIBentryALTinterwordspacing}{\spaceskip=\fontdimen2\font plus
\BIBentryALTinterwordstretchfactor\fontdimen3\font minus
  \fontdimen4\font\relax}
\providecommand{\BIBforeignlanguage}[2]{{%
\expandafter\ifx\csname l@#1\endcsname\relax
\typeout{** WARNING: IEEEtran.bst: No hyphenation pattern has been}%
\typeout{** loaded for the language `#1'. Using the pattern for}%
\typeout{** the default language instead.}%
\else
\language=\csname l@#1\endcsname
\fi
#2}}
\providecommand{\BIBdecl}{\relax}
\BIBdecl

\bibitem{Li_2010LinUCB}
\BIBentryALTinterwordspacing
L.~Li, W.~Chu, J.~Langford, and R.~E. Schapire, ``A contextual-bandit approach
  to personalized news article recommendation,'' \emph{Proceedings of the 19th
  international conference on World wide web - WWW ’10}, 2010. [Online].
  Available: \url{http://dx.doi.org/10.1145/1772690.1772758}
\BIBentrySTDinterwordspacing

\bibitem{DisplayAdvertising2017}
\BIBentryALTinterwordspacing
E.~M. Schwartz, E.~T. Bradlow, and P.~S. Fader, ``Customer acquisition via
  display advertising using multi-armed bandit experiments,'' \emph{Marketing
  Science}, vol.~36, no.~4, pp. 500--522, 2017. [Online]. Available:
  \url{https://doi.org/10.1287/mksc.2016.1023}
\BIBentrySTDinterwordspacing

\bibitem{bks}
R.~E. Bechhofer, J.~Kiefer, and M.~Sobel, ``Sequential identification and
  ranking procedures,'' \emph{The University of Chicago Press, Chicago-London},
  1968.

\bibitem{bk}
R.~E. Bechhofer and R.~V. Kulkarni, ``Closed adaptive sequential procedures for
  selecting the best of $k > 2$ bernoulli populations,'' in \emph{Statistical
  Decision Theory and Related Topics, Vol 1, Academic Press}, S.~Gupta and
  J.~O. Berger, Eds.\hskip 1em plus 0.5em minus 0.4em\relax New York: Cornell
  univ Ithaca NY School of Operations research and industrial engineering,
  1982, pp. 61--108.

\bibitem{h}
\BIBentryALTinterwordspacing
M.~Hartmann, ``An improvement on paulson s sequential ranking procedure,''
  \emph{Sequential Analysis}, vol.~7, no.~4, pp. 363--372, 1988. [Online].
  Available: \url{https://doi.org/10.1080/07474948808836163}
\BIBentrySTDinterwordspacing

\bibitem{hm}
D.~G. Hoel and M.~Mazumdar, ``An extension of paulson's selection procedure,''
  \emph{Ann Math. Statist.}, vol.~39, p. 1968, 1968.

\bibitem{p}
E.~Paulson, ``A sequential procedure for selecting the population with the
  largest mean from k normal populations,'' \emph{Ann. Math. Statist}, vol.~35,
  pp. 174--180, 1964.

\bibitem{sw1}
M.~Sobel and G.~Weiss, ``Play-the-winner rule and inverse sampling in selecting
  the better of two binomial populations,'' \emph{Journal of the American
  Statistical Association}, vol.~66, no. 335, pp. 545--551, 1971.

\bibitem{sw2}
------, ``Recent results on using the play the winner sampling rule with
  binomial selection problems,'' in \emph{Proceedings of the Sixth Berkeley
  Symposium on Mathematical Statistics and Probability, Vol 1; Theory of
  Statistics}, University of California Press, 1972, pp. 717--736.

\bibitem{abm}
J.~Y. Audibert, S.~Bubeck, and R.~Munos, ``Best arm identification in
  multi-armed bandits,'' in \emph{COLT 2010}.\hskip 1em plus 0.5em minus
  0.4em\relax Haifa, Israel: The 23rd Conference on Learning Theory, 2010, p.
  13 p.

\bibitem{jmnb}
K.~Jamieson, M.~Malloy, R.~Nowak, and S.~Bubeck, ``lil' ucb : An optimal
  exploration algorithm for multi-armed bandits,'' 2013.

\bibitem{ggl}
\BIBentryALTinterwordspacing
V.~Gabillon, M.~Ghavamzadeh, and A.~Lazaric, ``Best arm identification: A
  unified approach to fixed budget and fixed confidence,'' in \emph{Advances in
  Neural Information Processing Systems}, F.~Pereira, C.~J.~C. Burges,
  L.~Bottou, and K.~Q. Weinberger, Eds., vol.~25.\hskip 1em plus 0.5em minus
  0.4em\relax Curran Associates, Inc., 2012, pp. 3212--3220. [Online].
  Available:
  \url{https://proceedings.neurips.cc/paper/2012/file/8b0d268963dd0cfb808aac48a549829f-Paper.pdf}
\BIBentrySTDinterwordspacing

\bibitem{emm}
E.~Even-Dar, S.~Mannor, and Y.~Mansour, ``Action elimination and stopping
  conditions for the multi-armed bandit and reinforcement learning problems,''
  \emph{Journal of Machine Learning Research}, vol.~7, pp. 1079--1105, 2006.

\bibitem{gk}
A.~Garivier and E.~Kaufmann, ``Optimal best arm identification with fixed
  confidence,'' \emph{JMLR Workshop and Conference Proceedings}, vol.~49, pp.
  1--30, 2016.

\bibitem{ru}
D.~Russo, ``Simple bayesian algorithms for best arm identification,'' in
  \emph{CoRR}, abs/1602.08448, 2016.

\bibitem{degenne2020gamification}
R.~Degenne, P.~Ménard, X.~Shang, and M.~Valko, ``Gamification of pure
  exploration for linear bandits,'' 2020.

\bibitem{wang2021core}
N.~Wang, B.~Kveton, and M.~Karimzadehgan, ``Core: Capitalizing on rewards in
  bandit exploration,'' 2021.

\bibitem{bastani2020explorationfree}
H.~Bastani, M.~Bayati, and K.~Khosravi, ``Mostly exploration-free algorithms
  for contextual bandits,'' 2020.

\bibitem{ro}
S.~M. Ross, \emph{Stochastic Processes}, 2nd~ed.\hskip 1em plus 0.5em minus
  0.4em\relax Wiley, 1996.

\bibitem{Jamieson-2014}
K.~{Jamieson} and R.~{Nowak}, ``Best-arm identification algorithms for
  multi-armed bandits in the fixed confidence setting,'' in \emph{2014 48th
  Annual Conference on Information Sciences and Systems (CISS)}, 2014, pp.
  1--6.

\bibitem{kaufmann2013information}
E.~Kaufmann and S.~Kalyanakrishnan, ``Information complexity in bandit subset
  selection,'' in \emph{Conference on Learning Theory}, 2013, pp. 228--251.

\bibitem{russo2018simple}
D.~Russo, ``Simple bayesian algorithms for best arm identification,'' 2018.

\bibitem{bp}
R.~E. Barlow and F.~Proschan, \emph{Statistical Theory of Reliabilitiy and Life
  Testing: Probability Models}.\hskip 1em plus 0.5em minus 0.4em\relax Holt,
  Rinehart and Winston, 1975.

\end{thebibliography}

\newpage
\appendix
\addcontentsline{toc}{section}{Appendices}
\renewcommand{\thesubsection}{\Alph{subsection}}
\subsection{Proofs}
	\begin{proof}[Proof of \cref{lemma:Eprelim}]
	\begin{eqnarray*}
		E[e^{- 2 \mu W}|W > 0]=&  \frac{\int_0 ^{\infty} e^{- 2 \mu x} e^{- (x - \mu)^2/2} dx}{ \sqrt{ 2 \pi} P(W > 0)}= \\
		\frac{\int_0 ^{\infty} e^{-(x+\mu)^2/2}dx}{ \sqrt{ 2 \pi} \Phi(\mu)} =& \frac{1-\Phi(\mu)}{\Phi(\mu)}
	\end{eqnarray*}
	Let  $Z = W- \mu$.  
	\begin{eqnarray*}
		E[W|W > 0]=&  \mu + E[Z| Z > - \mu]=\\
		\mu + \frac{1}{ \sqrt{ 2 \pi} \Phi(\mu)}  \int_{ - \mu}^{\infty} x e^{- x^2/2} \, dx
		=&  \mu + \frac{e^{- \mu^2/2}}{ \sqrt{ 2 \pi} \Phi(\mu)    } 
	\end{eqnarray*}
	Because $E[e^{-2 \mu W}] = 1,$  the third equality follows from the first upon using the identity
	$$1 = E[e^{-2 \mu W}|W > 0] \Phi(\mu) + E[e^{-2 \mu W}|W < 0] (1 -\Phi(\mu) ) $$
	Similarly, the fourth equality follows from the second since  $\; \mu = E[W|W > 0]  \Phi(\mu) + E[W|W < 0] (1 -\Phi(\mu) ).$
	\end{proof}

	\begin{proof}[Proof \cref{lemma:ESn}]
	The right hand inequality of (a) is immediate since $\mu > 0.$  
	To prove the left side of (a), note that conditional on $S_N > b$ and on the value $S_{N-1},$ that  $S_N$  is distributed as $b$ plus the amount by which a normal with mean $\mu$ and variance $1$ exceeds the positive amount $b - S_{N-1}$ given that it does exceed that amount. But a normal conditioned to be positive is known to have strict increasing failure rate (see \cite{bp}) implying that 
	$S_N | \{S_N > b, S_{N-1} \}$ is stochastically smaller than $ b + W_i| \{W_i > 0\}$. As this is true no matter what the value of $S_{N-1},$ it follows that $S_N | \{S_N > b \}$ is stochastically smaller than $ b + W_i| \{W_i > 0\}$, implying that
	$E[e^{- 2 \mu S_N}| S_N > b]  > e^{-2 \mu b} E[e^{- 2 \mu W_i}|W_i > 0] .$ The result now  follows from Lemma \ref{lemma:Eprelim}.\\
	The left hand inequality of (b) is immediate.  
	To prove  the right hand inequality, note that the same argument as used in part (a) shows that  $S_N | \{S_N < - b\} >_{st} -b + W_i | \{W_i < 0\} ,$ implying that 
	$E[e^{- 2 \mu S_N}| S_N < -b ]  < e^{2 \mu b} E[e^{- 2 \mu W_i}|W_i < 0] .$ Thus, the result follows from Lemma \ref{lemma:Eprelim}.
	\end{proof}

\subsection{A remark on Variance Reduction}
	In our experiments for the VT rule, we observe that the estimator of $\E{N}$ has a large variance. 
	In the case where $F$ is the uniform $(0,1)$ distribution, we can reduce the variance of $\E{N}$ estimator by using $Y = \frac{1}{P_1(1-P_2)}$ as a control variable, where $P_1$ and $P_2$ are the random variables representing the means of the best and second best arm. That is, if let $T$ denote the raw estimator, then the new estimator is $ T + c(Y  - \E{Y})$
	where the variance is minimized when
	$c = -\mbox{Cov}(T,  Y)/\mbox{Var}(Y)$. To obtain the mean value of the control variable, we condition on $P_2$,
	\begin{align*}
		\E{\frac{1}{P_1(1 - P_2)}}= &\E{\E{\frac{1}{P_1(1 - P_2)} | P_2}}\\
		=\E{ \frac{1}{1-P_2} \E{\frac{1}{P_1}|P_2}}= & \E{\frac{-\log(P_2)}{(1 - P_2)^2}}=\\
		n(n-1)  \int_0^1 \frac{-x^{n-2}\log(x)}{(1-x)} dx
		\approx & n(n-1) \frac{1}{r} \sum_{i = 1} ^ r h(\frac{i - 0.5}{r}) 
	\end{align*}
	where $r$ is a large integer, and $h(x) =  \frac{-x^{n-2}\log(x)}{(1-x)}$. The third equality holds because $P_1|P_2 \sim \mbox{unif}(P_2, 1)$. The values of $\mbox{Cov}(T, Y)$ and $\mbox{Var}(Y)$ can be estimated from the simulation, and these can then be used to determine $c$. 
	In our numerical examples, we observe that the variance is reduced by up to $60$ percent using this technique.
	
\subsection{PW with Early Elimination}\label{sec:PW-EE}
	Suppose we use PW and add an early elimination on any population whose first $j$ observations are all failures. Let $B_e$ be the event that the best population is eliminated early.  Because the mean of the best population has density function $f(p) = n p^{n-1}, 0 < p < 1,$ it follows that 
	$$P(B_e) = \int_0 ^1 (1-p)^j n p^{n-1} dp =  \frac{ n! j!}{(n+j)!} $$
	Let $N_{nb}$ be the  number of nonbest populations that are eliminated early. To compute  $E[N_{nb}],$  note that the probability a randomly chosen population is eliminated early is $1/(j+1),$ giving that 
	$$\frac{n}{j+1} = E[\mbox{number eliminated early}] = E[N_{nb}] +  \frac{ n! j!}{(n+j)!} $$
	Hence, 
	$E[N_{nb}] = \frac{n}{j+1} - \frac{ n! j!}{(n+j)!}$. For instance, if $n=10, j=5$ then  $P(B_e) = 0.000333$ and  $E[N_{nb}]= 1.666$ .
	Remarkably, early elimination sometimes has almost no effect on either $P(C)$ or the average number of needed trials.
	\begin{ex}
	   Suppose $n=5$ and $k=48.$ Then, a  simulation with $20,000,000 $ runs yielded Table \ref{tab:PW-PW-EE}.
	\begin{table}[ht]
	\small
		\begin{center}
			\begin{tabular}{c c c c }
				\hline
				Rule & $P(C)$ & $E[N]$ & sd \\
				\hline
				PW    & 0.9543137 &  375.3552  &  0.1410 \\
				PW-EE &  0.9544769 &  375.4235 &  0.1411 \\
				\hline 
			\end{tabular}
		\end{center}
		\caption{Results of PW and PW with early elimination, PW-EE}
		\label{tab:PW-PW-EE}
	\end{table}
	Thus 
	it seems impossible to tell in this example whether early elimination  increases either accuracy or efficiency. (In particular, since the mean number of non-best populations that are eliminated early is $0.7121$ it seems very surprising that early elimination does not decrease the average number of trials needed.)
	\end{ex}

\end{document}